\documentclass{article}

\usepackage{microtype}
\usepackage{graphicx}
\usepackage{subfigure}
\usepackage{booktabs} 

\usepackage{amsmath,amsfonts,bm}

\def\eqref#1{equation~\ref{#1}}

\def\1{\bm{1}}

\DeclareMathAlphabet{\mathsfit}{\encodingdefault}{\sfdefault}{m}{sl}
\SetMathAlphabet{\mathsfit}{bold}{\encodingdefault}{\sfdefault}{bx}{n}

\usepackage{hyperref}
\usepackage{url}
\usepackage[utf8]{inputenc}
\usepackage[english]{babel}
 
\usepackage{amsthm}
\usepackage{amsfonts}       
\usepackage{nicefrac}       
\usepackage{amsmath}
\usepackage{bbm}
\usepackage{algorithm,algorithmic}
\usepackage{multirow}
\usepackage{bbm}
\usepackage[algo2e]{algorithm2e} 
\usepackage[colorinlistoftodos]{todonotes}
\newtheorem{proposition}{Proposition}
\newtheorem{repproposition}{Proposition}
\newtheorem{claim}{Claim}

\usepackage{hyperref}

\usepackage[accepted]{icml2020}

\icmltitlerunning{Improving the Tightness of Convex Relaxation Bounds for Training Certifiably Robust Classifiers}

\begin{document}

\twocolumn[
\icmltitle{Improving the Tightness of Convex Relaxation Bounds \\
for Training Certifiably Robust Classifiers}




\begin{icmlauthorlist}
\icmlauthor{Chen Zhu}{umd}
\icmlauthor{Renkun Ni}{umd}
\icmlauthor{Ping-yeh Chiang}{umd}
\icmlauthor{Hengduo Li}{umd}
\icmlauthor{Furong Huang}{umd}
\icmlauthor{Tom Goldstein}{umd}
\end{icmlauthorlist}

\icmlaffiliation{umd}{Department of Computer Science, University of Maryland, College Park}

\icmlcorrespondingauthor{Chen Zhu}{chenzhu@cs.umd.edu}
\icmlcorrespondingauthor{Tom Goldstein}{tomg@cs.umd.edu}

\icmlkeywords{Machine Learning, ICML}

\vskip 0.3in
]



\printAffiliationsAndNotice{}  

\begin{abstract}
Convex relaxations are effective for training and certifying neural networks against norm-bounded adversarial attacks, but they leave a large gap between certifiable and empirical robustness. In principle, convex relaxation can provide tight bounds if the solution to the relaxed problem is feasible for the original non-convex problem. We propose two regularizers that can be used to train neural networks that yield tighter convex relaxation bounds for robustness. In all of our experiments, the proposed regularizers result in higher certified accuracy than non-regularized baselines. 
\end{abstract}

\newcommand{\cI}{\mathcal{I}}
\newcommand{\cS}{\mathcal{S}}
\newcommand{\conv}{\overline{conv}}
\newcommand{\zij}{z_{ij}}
\newcommand{\ziplj}{z_{i+1,j}}
\newcommand{\xij}{x_{ij}}
\newcommand{\lxij}{\underline{x}_{ij}}
\newcommand{\uxij}{\overline{x}_{ij}}
\newcommand{\lxiprimej}{\underline{x}_{i'j}}
\newcommand{\uxiprimej}{\overline{x}_{i'j}}
\newcommand{\lxnj}{\underline{x}_{nj}}
\newcommand{\uxnj}{\overline{x}_{nj}}
\newcommand{\uaij}{\overline{a}_{ij}}
\newcommand{\laij}{\underline{a}_{ij}}
\newcommand{\lbij}{\underline{b}_{ij}}
\newcommand{\ubij}{\overline{b}_{ij}}
\newcommand{\cW}{\mathcal{W}}
\newcommand{\popt}{\mathcal{O}}
\newcommand{\ropt}{\mathcal{C}}
\newcommand{\slope}{\frac{\uxij}{\uxij - \lxij}}
\newcommand{\udelta}{-\frac{\uxij\lxij}{\uxij - \lxij}}
\newcommand{\uxtwoj}{\overline{x}_{1,j}}
\newcommand{\lxtwoj}{\underline{x}_{1,j}}
\newcommand{\slopetwo}{\frac{\uxtwoj}{\uxtwoj - \lxtwoj}}
\newcommand{\udeltatwo}{-\frac{\uxtwoj\lxtwoj}{\uxtwoj - \lxtwoj}}
\newcommand{\lxi}{\underline{x}_{i}}
\newcommand{\uxi}{\overline{x}_{i}}
\newcommand{\lxis}{\underline{x}_{i-1}}
\newcommand{\uxis}{\overline{x}_{i-1}}
\newcommand{\regd}{ d(x,\delta_0^*,W,b)}
\newcommand{\regr}{r(x,\delta_0^*, W, b)}
\newcommand{\real}{\mathbb{R}}
\section{Introduction}
Neural networks have achieved excellent performances on many computer vision tasks, but they are often vulnerable to small, adversarially chosen perturbations that are barely perceptible to humans while having a catastrophic impact on model performance~\citep{szegedy2013intriguing,goodfellow2014explaining}. Making classifiers robust to these adversarial perturbations is of great interest, especially when neural networks are applied to safety-critical applications. 
Several heuristic methods exist for obtaining robust classifiers, however powerful adversarial examples can be found against most of these defenses~\citep{carlini2017adversarial,uesato2018adversarial}. 

Recent studies focus on verifying or enforcing the certified accuracy of deep classifiers, especially for networks with ReLU activations. 
They provide guarantees of a network's robustness to any perturbation $\delta$ with norm bounded by $\lVert \delta \rVert_p\le \epsilon$~\citep{wong2017provable,wong2018scaling,raghunathan2018certified,dvijotham2018dual,zhang2018efficient, salman2019convex}. 
Formal verification methods can find the exact minimum adversarial distortions needed to fool a classifier~\citep{ehlers2017formal,katz2017reluplex,tjeng2017evaluating}, but require solving an NP-hard problem. 
To make verification efficient and scalable, convex relaxations are adopted, resulting in a lower bound on the norm of adversarial perturbations~\citep{zhang2018efficient, weng2018towards},
or an upper bound on the robust error~\citep{dvijotham2018dual,gehr2018ai2,singh2018fast}.
Linear programming (LP) relaxations~\citep{wong2018scaling} are efficient enough to estimate the lower bound of the margin in each iteration for training certifiably robust networks. 
However, due to the relaxation of the underlying problem, a wide gap remains between the optimal values from the original and relaxed problems~\citep{salman2019convex}.  

In this paper, we focus on improving the certified robustness of neural networks trained with convex relaxation bounds. 
To achieve this, we first give a more interpretable explanation for the bounds achieved in~\citep{weng2018towards,wong2018scaling}. 
Namely, the constraints of the relaxed problem are defined by a simple linear network with adversaries injecting bounded perturbations to both the input of the network and the pre-activations of intermediate layers.
The optimal solution of the relaxed problem can be written as a forward pass of the clean image through the linear network, plus the cumulative adversarial effects of all the perturbations added to the linear transforms, which makes it easier to identify the optimality conditions and serves as a bridge between the relaxed problem and the original non-convex problem.
We further identify conditions for the bound to be tight, and 
we propose two indicators for the gap between the original non-convex problem and the relaxed problem. Adding the proposed indicators into the loss function results in classifiers with better certified accuracy.

\section{Background and Related Work}
Adversarial defenses roughly fall into two categories: heuristic defenses and verifiable defenses. The heuristic defenses either try to identify adversarial examples and remove adversarial perturbations from images, or make the network invariant to small perturbations through training \citep{papernot2018deep,shan2019gotta,samangouei2018defense,hwang2019puvae}.
In addition, adversarial training uses adversarial examples as opposed to clean examples during training, so that the network can learn how to classify adversarial examples directly~\citep{madry2017towards, shafahi2019adversarial, zhang2019you}. 

In response, a line of works have proposed to verify the robustness of neural nets. 
Exact methods obtain the perturbation $\delta$ with minimum $\lVert \delta \rVert_p$ such that $f(x) \neq f(x+\delta)$, where $f$ is a classifier and $x$ is the data point. 
Nevertheless, the problem itself is NP-hard and the methods can hardly scale~\citep{cheng2017maximum,lomuscio2017approach,dutta2018output,fischetti2017deep,tjeng2017evaluating,scheibler2015towards,katz2017reluplex,carlini2017provably,ehlers2017formal}.

A body of work focuses on relaxing the non-linearities in the original problem into linear inequality constraints~\citep{singh2018fast,gehr2018ai2,zhang2018efficient,mirman2018differentiable}, sometimes using the dual of the relaxed problem~\citep{wong2017provable,wong2018scaling,dvijotham2018dual}.
Recently,~\cite{salman2019convex} unified the primal and dual views into a common convex relaxation framework, and suggested there is an inherent gap between the actual and the lower bound of robustness given by verifiers based on LP relaxations, which they called a \textit{convex relaxation barrier}.    

Some defense approaches integrate the verification methods into the training of a network to minimize robust loss directly. 
\cite{hein2017formal} uses a local lipschitz regularization to improve certified robustness. 
In addition, a bound based on semi-definite programming (SDP) relaxation was developed and minimized as the objective~\citep{raghunathan2018certified}.
~\cite{wong2017provable} presents an upper bound on the robust loss caused by norm-bounded perturbation via LP relaxation, and minimizes this upper bound during training. 
\cite{wong2018scaling} further extend this method to much more general network structures with skip connections and general non-linearities, and provide a memory-friendly training strategy using random projections. 
Since LP relaxation is adopted, the aforementioned convex relaxation barrier exists for their methods.

While another line of work (IBP) have shown that an intuitively looser interval bound can be used to train much more robust networks than convex relaxation for large $\ell_\infty$ perturbations~\citep{gowal2018effectiveness, zhang2019towards}, it is still important to study convex relaxation bounds since it can provide better certificates against a broader class of adversaries that IBP struggles to certify in some cases, such as $\ell_2$ adversaries for convolutional networks. We discuss these motivations in more detail in Appendix~\ref{sec:ibp_l2}.

We seek to enforce the tightness of the convex relaxation certificate during training. 
We reduce the optimality gap between the original and the relaxed problem by using various tightness indicators as regularizers during training.
Compared with previous approaches, we have the following contributions: First, based upon the same relaxation in~\citep{weng2018towards}, we illustrate a more intuitive view for the bounds on intermediate ReLU activations achieved by~\citep{wong2018scaling}, which can be viewed as a linear network facing adversaries adding that make bounded perturbations to both the input and the intermediate layers. 
Second, starting from this view, we identify conditions where the bound from the relaxed problem is tight for the original non-convex problem.
Third, based on the conditions, we propose regularizers that encourage the bound to be tight for the obtained network, which improves the certificate on both MNIST and CIFAR-10.

\section{Problem Formulation}\label{sec:method}
In general, to train an adversarially robust network, we solve a constrained minimax problem where the adversary tries to maximize the loss given the norm constraint, and the parameters of the network are trained to minimize this maximal loss. 
Due to nonconvexity and the complexity of neural networks, it is expensive to solve the inner max problem exactly. 
To obtain certified robustness, like many related works~\citep{wong2018scaling, gowal2018effectiveness}, we minimize an upper bound of the inner max problem, which is a cross entropy loss on the negation of the lower bounds of margins over each other class, as shown in Eq.~\ref{eq:obj}.
Without loss of generality, in this section we analyze the original and relaxed problems for minimizing the margin between the ground truth class $y$ and some other class $t$ under norm-bounded adversaries, which can be adapted directly to compute the loss in Eq.~\ref{eq:obj}.

The original nonconvex constrained optimization problem for finding the norm-bounded adversary that minimizes the margin can be formulated as
\begin{align}\label{eq:nonconvex}\tag{$\popt$}
\underset{z_1\in\mathcal{B}_{p,\epsilon}(x)}{\text{minimize }} & c_t^\top x_{L}, \\
\text{subject to } & z_{i+1}=\sigma(x_{i}), x_{i}=f_i(z_i),  \nonumber\\
\text{for}~ &i=1,...,L,\nonumber
\end{align}
 where $c_t=e_y-e_t$, $e_y$ and $e_t$ are one-hot vectors corresponding to the label $y$ and some other class $t$, 
$\sigma(\cdot)$ is the ReLU activation, and
$f_i$ is one functional block of the neural network.
This can be a linear layer ($f_i(z_i)=W_iz_i+b_i$), or even a residual block. 
We use $h_i(x)=f_i(\sigma(f_{i-1}(\cdots f_1(x))))$ to denote the ReLU network up to the $i$-th layer, and $p_\popt^*$ to denote the optimal solution to~\ref{eq:nonconvex}.

\begin{figure*}[!htbp]
\centering
  \includegraphics[width=0.75\linewidth]{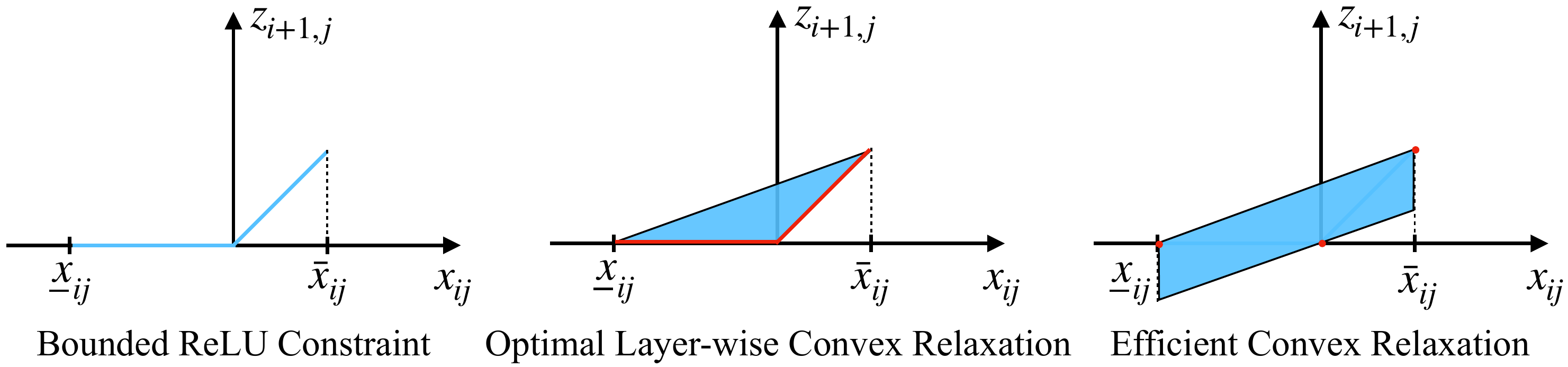}
  \caption{The feasible sets (blue regions/lines) given by the bounded ReLU constraints (Eq.~\ref{eq:nonconvex}), convex hull ($\conv_{ij}$) and the relaxation (Fast-Lin) discussed in this paper (specific choice for Eq.~\ref{eq:convex_relaxation}) for $j\in \cI_i$. The red lines and dots are the intersections between the boundaries of the convex feasible sets and the ReLU constraints.
  }
  \label{fig:relaxation} 
\end{figure*}


\subsection{Efficient Convex Relaxations}
\paragraph{Grouping of ReLU Activations}
The nonconvexity of~\ref{eq:nonconvex} stems from the nonconvex feasible set given by the ReLU activations.
Since the network is a continuous function, the pre-activations $x_i$ have lower and upper bounds $\underline{x}_i$ and $\bar{x}_i$ when the input $z_1\in\mathcal{B}_{p,\epsilon}(x)$. 
If a certain pre-activation $x_{ij}$ has $\underline{x}_{ij}<0<\bar{x}_{ij}$, its corresponding ReLU constraint $z_{i+1,j}=\sigma(x_{ij})$ gives rise to a non-convex feasible set as shown in the left of Figure~\ref{fig:relaxation}, making Eq.~\ref{eq:nonconvex} a non-convex optimization problem.
On the other hand, if $\bar{x}_{ij}\le 0$ or $\underline{x}_{ij}\ge 0$, the constraints degenerate into linear constraints $z_{i+1,j}=0$ and $z_{i+1,j}=x_{ij}$ respectively, which do not affect convexity. 
Based on $\underline{x}_i$ and $\bar{x}_i$, we divide the ReLU activations into three disjoint subsets
\begin{align}
    &\cI_i^{-} = \{j|\bar{x}_{ij} \le 0\},~~~\cI_i^{+} = \{j|\underline{x}_{ij} \ge 0\},\nonumber\\
    &\cI_i = \{j|\underline{x}_{ij} < 0 < \bar{x}_{ij}\}.
\end{align}
If $j\in \cI_i$, we call the corresponding ReLU activation an {\em unstable neuron}.

Convex relaxation expands the non-convex feasible sets into convex ones and solves a convex optimization problem $\mathcal{C}$. 
The feasible set of~\ref{eq:nonconvex} is a subset of the feasible set of $\mathcal{C}$, so the optimal value of $\mathcal{C}$ lower bounds the optimal value of Eq.~\ref{eq:nonconvex}. 
Moreover, we want problem $\mathcal{C}$ to be solved efficiently, better with a closed form solution, so that it can be integrated into the training process. 

\paragraph{Computational Challenge for the ``optimal'' Relaxation} As pointed out by~\citep{salman2019convex}, the optimal layer-wise convex relaxation, i.e., the optimal convex relaxation for the nonlinear constraint $z_{i+1}=\sigma(x_i)$ of a single layer, can be obtained independently for each neuron. 
For each $j\in \cI_i$ in a ReLU network, the optimal layer-wise convex relaxation is the closed convex hull $\conv_{ij}$ of $\cS_{ij}=\{(x_{ij}, z_{i+1,j})|j\in \cI_i, z_{i+1,j}=\max(0,\xij), \lxij \le \xij \le \uxij\}$, which is just $\conv_{ij}=\{ (x_{ij}, z_{i+1,j}) | \max(0, \xij) \le z_{i+1,j} \le \frac{\uxij}{\uxij - \lxij}(\xij-\lxij)\}$, corresponding to the triangle region in the middle of Figure~\ref{fig:relaxation}.
Despite being relatively tight, there is no closed-form solution to this relaxed problem. 
LP solvers are typically adopted to solve a linear programming problem for each neuron.
Therefore, such a relaxation is hardly scalable to verify larger networks without any additional trick (like~\cite{xiao2018training}). 
\cite{weng2018towards} find it to be 34 to 1523 times slower than Fast-Lin, and it has difficulty verifying MLPs with more than 3 layers on MNIST.
In~\citep{salman2019convex}, it takes 10,000 CPU cores to parallelize the LP solvers for bounding the activations of every neuron in a two-hidden-layer MLP with 100 neurons per layer.
Since solving LP problems for all neurons are usually impractical, it is even more difficult to optimize the network to maximize the lower bounds of margin found by solving this relaxation problem, as differentiating through the LP optimization process is even more expensive.

\paragraph{Computationally Efficient Relaxations}
In the layer-wise convex relaxation, instead of using a boundary nonlinear in $x_{ij}$,~\citep{zhang2018efficient} has shown that for any nonlinearity, when both the lower and upper boundaries are linear in $x_{ij}$, there exist closed-form solutions to the relaxed problem, which avoids using LP solvers and improves efficiency. 
Specifically, the following relaxation of~\ref{eq:nonconvex} has closed-form solutions:
\begin{align}\label{eq:convex_relaxation}\tag{$\mathcal{C}$}
    \text{minimize }~~ &  c_t^\top x_{L}, \nonumber\\
    \text{subject to }~~ & \underline{a}_i\cdot x_i + \underline{b}_i \le z_{i+1}\le \bar{a}_i \cdot x_i + \bar{b}_i,\nonumber\\
    &\underline{x}_i \le x_i \le \bar{x}_i, \nonumber \\
    &x_{i}=W_i z_i + b_i, \nonumber \\
    &\text{for } i=1,...,L, z_1\in\mathcal{B}_{p,\epsilon}(x),\nonumber
\end{align}
where $\cdot$ denotes element-wise product, and for simplicity, we have only considered networks with no skip connections, and represent both Full Connected and Convolutional Layers as a linear transform $f_i(z_i)=W_i z_i+b_i$.

Before we can solve~\ref{eq:convex_relaxation} to get the lower bound of margin, we need to know the range $[\underline{x}_i, \bar{x}_i]$ for the pre-activations $x_i$. 
As in~\citep{wong2017provable, weng2018towards,zhang2018efficient}, we can solve the same optimization problem for each neuron $x_{ij}$ starting from layer 1 to $L$, by replacing $c_t$ with $e_j$ or $-e_j$ for $\underline{x}_{ij}$ or $\bar{x}_{ij}$ respectively.\footnote{For $\bar{x}_{ij}$, take an extra negation on the solution.}

The most efficient approach in this category is Fast-Lin~\citep{weng2018towards}, which sets $\underline{a}_{ij}=\bar{a}_{ij}$, as shown in the right of Figure~\ref{fig:relaxation}. 
A tighter choice is CROWN~\citep{zhang2018efficient}, which chooses different $\underline{a}_{ij}$ and $\bar{a}_{ij}$ such that the convex feasible set is minimized. 
However, CROWN has much higher complexity than Fast-Lin due to its varying slopes.
We give detailed analysis of the closed-form solutions of both bounds and their complexities in Appendix~\ref{sec:cr_sols}.
Recently, CROWN-IBP~\citep{zhang2019towards} has been proposed to provide a better initialization to IBP, which uses IBP to estimate range $[\underline{x}_i, \bar{x}_i]$ for CROWN.
In this case, both CROWN and Fast-Lin have the same complexity and CROWN is a better choice. 



\begin{figure}[!htbp]
\centering
  \includegraphics[width=.6\linewidth]{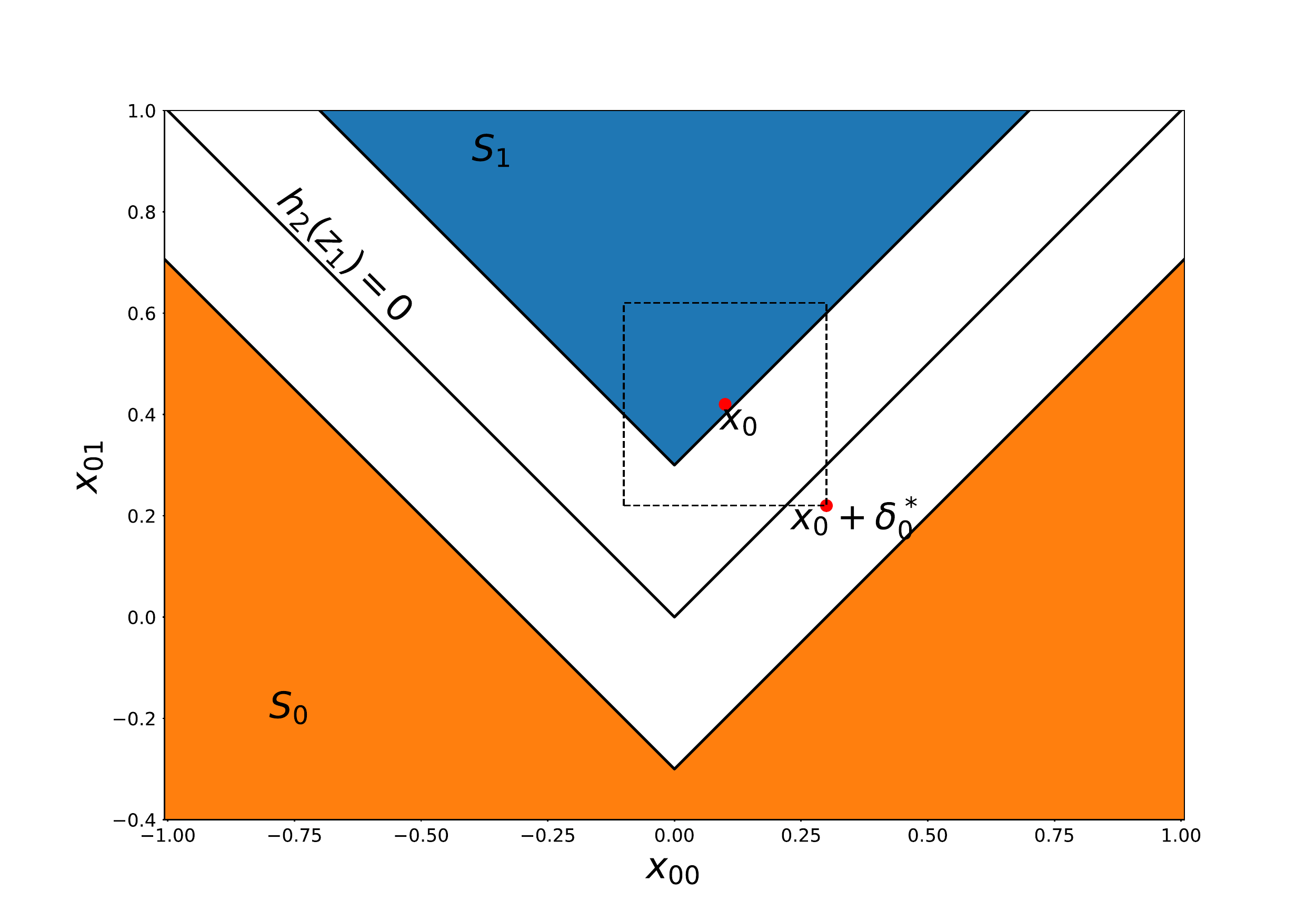}
  \caption{Illustration of the data distribution and the decision boundary of the network $h_2(z_1)$, where Fast-Lin gives the exact lower bound of the margin for every sample in $S_1$. We assume $x_0$ is uniformly distributed in $S_0\cup S_1$, where $S_0=\{x_0| x_{02}\le |x_{01}|-b, \lVert x_0 \rVert_{\infty}\le 1\}$, $S_1=\{x_0| x_{02}\ge |x_{01}|+b,  \lVert x_0 \rVert_{\infty}\le 1\}$, and $0<b<1$.
The ground-truth label for $x_0\in S_0$ and $x_0\in S_1$ are 0, 1 respectively. In this case, $b=0.3, \epsilon=0.2, x_0=[0.1, 0.42]^T$. 
  }
  \label{fig:toy1} 
\end{figure}

\section{Tighter Bounds via Regularization}\label{sec:indicator}
Despite being relatively efficient to compute, Fast-Lin and CROWN are not even the tightest layer-wise convex relaxation. 
Using tighter bounds to train the networks could potentially lead to higher certified robustness by preventing such bounds from over-regularizing the networks. 

Nevertheless, there exist certain parameters and inputs such that the seemingly looser Fast-Lin is tight for~\ref{eq:nonconvex}, i.e., the optimal value of Fast-Lin is the same as~\ref{eq:nonconvex}. 
The immediate trivial case on can think of is where no unstable neuron exists for the samples inside the allowed perturbation interval. 

In fact, even when unstable neurons exist, the optimal solution to the relaxed problem can still be a feasible solution to the original non-convex problem for a significant portion of input samples . 
We give an illustrative example where Fast-Lin is tight for a significant portion of the samples even when unstable neurons exist, as shown in Figure~\ref{fig:toy1}.
In this figure, Fast-Lin is tight at every sample of $S_1$ for the network $h_1(z_2)$.
Please refer to Appendix~\ref{sec:toy} for more details of this example. 

It is therefore interesting to check the conditions for Fast-Lin or CROWN to be tight for~\ref{eq:nonconvex}, and enforcing such conditions during training so that the network can be better verified by efficient verifiers like Fast-Lin, CROWN, and even IBP.

\subsection{Conditions for Tightness}
Here we look into conditions that make the optimal value $p_{\ropt}^*$ of the convex problem $\mathcal{C} $ to be equal to $p_{\popt}^*$. 
Let $\{z_i, x_i\}_{i=1}^L$ be some feasible solution of $\mathcal{C}$, from which the objective value of $\mathcal{C}$ can be determined as $p_{\mathcal{C}}=c_t^\top x_L$. 
Let $\{z'_i, x'_i\}_{i=1}^{L}$ be some feasible solution of $\popt$ computed by passing $z'_1$ through the ReLU sub-networks $h_i(z'_1)$ defined in~\ref{eq:nonconvex}, and denote the resulting feasible objective value as $p'_{\popt}=c_t^\top x'_L$.

Generally, for a given network with the set of weights $\{W_i,b_i\}_{i=1}^L$, as long as the optimal solution $\{ z_i^*, x_i^* \}_{i=1}^L$ of $\mathcal{C}$ is equal to a feasible solution $\{z'_i, x'_i\}_{i=1}^L$ of~\ref{eq:nonconvex}, we will have $p^*_{\popt}=p^*_{\ropt}$,
since any feasible $p'_{\mathcal{O}}$ of $\mathcal{O}$ satisfies $p'_{\mathcal{O}}\ge p_{\mathcal{O}}^*$, and by the nature of relaxation $p_{\mathcal{C}}^* \le p_{\mathcal{O}}^*$.

Therefore, for a given network and input $x$, to check the tightness of the convex relaxation, we can check whether its optimal solution $\{ z_i^*, x_i^* \}_{i=1}^L$ is feasible for $\mathcal{O}$. 
This can be achieved by passing $z_1^*$ through the ReLU network, and either directly check the resultant objective value $p_{\mathcal{O}}'$, or compare $\{ z_i^*, x_i^* \}_{i=1}^L$ with the resultant feasible solution $\{z'_i, x'_i\}_{i=1}^L$. 
Further, we can encourage such conditions to happen during the training process to improve the tightness of the bound.
Based on such mechanisms, we propose two regularizers to enforce the tightness. 
Notice such regularizers are different from the RS Loss~\citep{xiao2018training} introduced to reduce the number of unstable neurons, since we have shown with Appendix~\ref{sec:toy} that $\mathcal{C}$ can be tight even when unstable neurons exist.

\subsection{A Intuitive Indicator of Tightness: Difference in Output Bounds}\label{sec:direct_indicator}
The observation above motivates us to consider the non-negative value
\begin{equation}\label{eq:distillation}
    d(x,\delta_0^*,W,b) = p'_{\popt}(x,\delta_0^*)-p_{\ropt}^*
\end{equation}
as an indicator of the difference between $\{z^*_i, x^*_i\}_{i=1}^{L}$ and $\{z'_i, x'_i\}_{i=1}^{L}$, 
where $p'_{\popt}(x,\delta_0^*)=c_t^\top h_L(x+\delta_0^*)$ is the margin over class $t$ computed by passing the optimal perturbation $\delta_0^*$ for $\mathcal{C}$ through the original network. 
$\delta_0^*$ can be computed efficiently from the optimality condition of Fast-Lin or CROWN, as demonstrated in Eq.~\ref{eq:dual_norm}.
For example, when $p=\infty$, the optimal input perturbation $\delta_0^*$ of $\ropt$ is $\delta_0^*=-\epsilon\text{sign}(c_t^\top\cW_{L:1})$, which corresponds to sending $z'_1=z^*_1=x-\epsilon \text{sign}(c_t^\top \cW_{L:1})$ through the ReLU network;
when $p=2$, $\delta_0^*=-\epsilon\frac{c_t^\top\cW_{L:1}}{\lVert c_t^\top\cW_{L:1} \rVert_2}$, which corresponds to sending $z'_1=z^*_1=x-\epsilon \frac{c_t^\top\cW_{L:1}}{\lVert c_t^\top\cW_{L:1} \rVert_2}$. 

The larger $d(x,\delta_0^*,W,b)$ is, the more relaxed $\ropt$ is, and the higher $p^*_{\popt}-p^*_{\ropt}$ could be. 
Therefore, we can regularize the network to minimize $ d(x,\delta_0^*,W,b)$ during training and maximize the lower-bound of the margin $p^*_{\ropt}$, so that we can obtain a network where $p^*_{\ropt}$ is a better estimate of $p^*_{\popt}$ and the robustness is better represented by $p^*_{\ropt}$.
Such an indicator avoids comparing the intermediate variables, which gives more flexibility for adjustment. 
It bears some similarities to knowledge distillation~\citep{hinton2015distilling}, in that it encourages learning a network whose relaxed lower bound gives similar outputs of the corresponding ReLU network.
It is worth noting that minimizing $d(x,\delta_0^*,W,b)$ does not necessarily lead to decreasing $ p'_{\popt}(x,\delta_0^*)$ or increasing $p_{\ropt}^*$. 
In fact, both $ p'_{\popt}(x,\delta_0^*)$ and $p_{\ropt}^*$ can be increased or decreased at the same time with their difference decreasing.

\begin{algorithm*}[htbp!]
 \KwData{Clean images$\{x^{(j)}\}_{j=1}^m$, ReLU network with parameters $W,b$, and maximum perturbation $\epsilon$.}
 \KwResult{Margin Lower Bound $p_{\ropt}^*$, Regularizer $d(x,\delta_0^*,W,b)$ and $r(x,\delta_0^*, W, b)$.}
 ~~Initialize $\cW_{1:1}\leftarrow W_1$, $g_1(x)=W_1 x + b_1$, $\underline{x}_1=g_1(x)-\epsilon \lVert \cW_{1:1} \rVert_{1,row}$, 
 $\bar{x}_1=g_1(x)+\epsilon \lVert \cW_{1:1} \rVert_{1,row}$\;

 \For{i=2,\ldots,L-1}{
  ~~Compute $D_{i-1}$ with Eq.~\ref{eq:Danddelta}, $\cW_{i:1}\leftarrow W_i D_{i-1} \cW_{i-1}, g_i (x)=W_i D_{i-1} g_{i-1}(x)+b_i$\\ 
  $\lxi\leftarrow g_i (x) - \epsilon \lVert \cW_{i:1} \rVert_{1,row}-\sum_{i'=1}^{i-1}\sum_{j\in I_{i'}}\frac{\lxiprimej\uxiprimej}{\uxiprimej - \lxiprimej}\min((\cW_{i:i'+1})_{:,j}, 0)$\;\\
  $\uxi\leftarrow g_i (x) + \epsilon \lVert \cW_{i:1} \rVert_{1,row}+\sum_{i'=1}^{i-1}\sum_{j\in I_{i'}}\frac{\lxiprimej\uxiprimej}{\uxiprimej - \lxiprimej}\min(-(\cW_{i:i'+1})_{:,j}, 0)$\;\\
 }
 
 ~~Compute $D_{L-1}$ with Eq.~\ref{eq:Danddelta}, $\cW_{L:1}\leftarrow W_L D_{L-1} \cW_{L-1}, g_L (x)=W_L D_{L-1} g_{L-1}(x)+b_L$\;\\
 Compute $p^*_{\ropt}$ with Eq.~\ref{eq:pc}, $\delta_0^*\leftarrow -\epsilon \lVert c_t^\top \cW_{L:1} \rVert_{1,row}$\;\\
 $x'_i\leftarrow h_i(x+\delta_0^*)$  \textbf{for} $i=1$ to $L$\;\\
 $d(x,\delta_0^*,W,b)\leftarrow c_t^\top x'_L - p^*_{\ropt}$\;\\
 Compute $r(x,\delta_0^*, W, b)$ with Eq.~\ref{eq:unstable_gap}
 \caption{Computing the Fast-Lin Bounds and Regularizers for $\ell_\infty$ Norm}
 \label{alg:opt}
\end{algorithm*}

The tightest indicator should give the minimum gap $p^*_{\popt}-p^*_{\ropt}$, where we need to find the optimal perturbation for~\ref{eq:nonconvex}. 
However, the minimum gap cannot be found in polynomial time, due to the non-convex nature of~\ref{eq:nonconvex}.
\citep{weng2018towards} also proved that there is no polynomial time algorithm to find the minimum $\ell_1$-norm adversarial distortion with $0.99\ln n$ approximation ratio unless NP=P, a problem equivalent to finding the minimum margin here.

\subsection{A Better Indicator for Regularization: Difference in Optimal Pre-activations}
Despite being intuitive and is able to achieve improvements, Eq.~\ref{eq:distillation} which enforces similarity between objective values does not work as good as enforcing similarity between the solutions $\{z^*_i, x^*_i\}_{i=1}^{L}$ and $\{z'_i, x'_i\}_{i=1}^{L}$ in practice, an approach we will elaborate below.
For both CROWN and Fast-Lin, unless $d(x,\delta_0^*,W,b)=0$, $\{z^*_i, x^*_i\}_{i=1}^{L}$ may deviate a lot from $\{z'_i, x'_i\}_{i=1}^{L}$ and does not correspond to any ReLU network, even if $d(x,\delta_0^*,W,b)$ may seem small.
For example, it is possible that $z^*_{ij}<0$ for a given $z^*_1$, but a ReLU network will always have $z'_{ij}\ge 0$.

We find an alternative regularizer more effective at improving verifiable accuracy. The regularizer encourages the feasible solution $\{z'_i, x'_i\}_{i=1}^{L}$ of~\ref{eq:nonconvex} to {\em exactly} match the feasible optimal solution $\{z^*_i, x^*_i\}_{i=1}^{L}$ of $\mathcal{C}$.
Since we are adopting the layer-wise convex relaxation, the optimal solutions of the unstable neurons can be considered independently. 

Here we derive a sufficient condition for tightness for Fast-Lin, which also serves as a sufficient condition for CROWN. 
For linear programming, the optimal solution occurs on the boundaries of the feasible set. 
Since Fast-Lin is a layer-wise convex relaxation, the solution to each of its neurons in $z_i$ can be considered independently, and therefore for a specific layer $i$ and $j\in \cI_i$, the pair of optimal solutions $(x^*_{ij}, z^*_{i+1,j})$ should occur on the boundary in the right of Figure~\ref{fig:relaxation}. 
It follows that the only 3 optimal solutions $(x_{ij}^*, z_{i+1,j}^*)$ of $\mathcal{C}$ that are also feasible for~\ref{eq:nonconvex} are $(\lxij, 0), (\uxij, \uxij)$ and $(0,0)$.
Notice they are also in the intersection between the boundary of CROWN and \ref{eq:nonconvex}.

In practice, out of efficiency concerns, both Fast-Lin and CROWN identify the boundaries that the optimal solution lies on and computes the optimal value by accumulating the contribution of each layer in a backward pass, without explicitly computing $\{z^*_i, x^*_i \}_{i=1}^L$ for each layer with a forward pass (see Appendix~\ref{sec:cr_sols} for more details).
It is therefore beneficial to link the feasible solutions of~\ref{eq:nonconvex} to the parameters of the boundaries. 
Specifically, let $\delta^*_{ij}\in \{\underline{b}_{ij}, \bar{b}_{ij}\}$ be the intercept of the line that the optimal solution $(x_{ij}^*, z_{i+1,j}^*)$ lies on.
We want to find a rule based on $\{\delta^*_i\}_{i=1}^L$ to determine whether the bound is tight from the values of $\{x'_i\}_{i=1}^L$.
For both Fast-Lin and CROWN, $\underline{b}_{ij} = 0, \bar{b}_{ij}=\udelta$.
For Fast-Lin, when $\delta^*_{ij}=\bar{b}_{ij}$, only $(x_{ij}^*, z_{i+1,j}^*)=(\lxij, 0)$ or $(\uxij, \uxij)$ are fesible for~\ref{eq:nonconvex}; when $\delta^*_{ij}=\underline{b}_{ij}$, only $(x_{ij}^*, z_{i+1,j}^*)=(0, 0)$ is feasible for~\ref{eq:nonconvex}.
Meanwhile, $z'_{i+1,j}=\max(x'_{ij},0)$ is deterministic if $x'_{ij}$ is given. 
Therefore, when the bound is tight for Fast-Lin, if $\delta^*_{ij}=\underline{b}_{ij}$, then $x'_{ij}=0$.
Otherwise, if $\delta^*_{ij}=\bar{b}_{ij}$, and $x'_{ij}=\lxij$ or $\uxij$.
For CROWN, this condition is also feasible, though it could be either $x'_{ij}\le 0$ or $x'_{ij}\ge 0$ when $\delta^*_{ij}=0$, depending on the optimal slope $\underline{D}^{(L)}_{ij}$.


Indeed, we achieve optimal tightness ($p_{\mathcal{C}}^*=p_{\mathcal{O}}^*$) for both Fast-Lin and CROWN if $\xij'$ satisfy these conditions at \emph{all} unstable neurons. 
Specifically, 

\begin{proposition}\label{prop:opt}
Assume $\{z_i', x_i'\}_{i=1}^L$ is obtained by the ReLU network $h_L$ with input $z_1'$, and $\{\delta_{i}^*\}_{i=0}^{L-1}$ is the optimal solution of Fast-Lin or CROWN. If $z'_1=x+\delta_0^*$, and $x'_{ij}\in \mathcal{S}(\delta^*_{ij})$ for all $i=1,...,L-1, j\in \cI_i$, then $\{z_i', x_i'\}_{i=1}^{L}$ is an optimal solution of ~\ref{eq:nonconvex}, Fast-Lin and CROWN. 
Here 
\begin{equation*}
 \mathcal{S}(\delta^*_{ij})=\begin{cases}
 \{\lxij, \uxij\} & \text{if } \delta^*_{ij} = \udelta,\\
 \{0\} & \text{if } \delta^*_{ij} = 0.
 \end{cases}
\end{equation*}
\end{proposition}

We provide the proof of this simple proposition in the Appendix.

It remains to be discussed how to best enforce the similarity between the optimal solutions of~\ref{eq:nonconvex} and Fast-Lin or CROWN. 
Like before, we choose to enforce the similarity between $\{x'_i\}_{i=1}^{L}$ and the closest optimal solution of Fast-Lin, where $\{x'_i\}_{i=1}^{L}$ is constructed by setting $x'_1=x^*_1=W_1(x+\delta_0^*)+b_1$ and pass $x'_1$ through the ReLU network to obtain $x'_i=h_i(x+\delta^*_0)$.
By Proposition~\ref{prop:opt}, the distance can be computed by considering the values of the intercepts $\{\delta^*_i\}_{i=1}^{L-1}$ as
\begin{equation}\label{eq:unstable_gap}
    \begin{split}
    r(x,\delta_0^*, W, b) &= \frac{1}{\sum_{i=1}^{L-1}|\cI_i|}\sum_{i=1}^{L-1}\left(\sum_{\substack{j\in \cI_i\\\delta_{ij}^*=0}} |x'_{ij}| \right. \nonumber\\
    &\quad \left. +\sum_{\substack{j\in \cI_i\\\delta_{ij}^*\neq 0}} \min(|x'_{ij}-\lxij|, |x'_{ij}-\uxij|)\right),
    \end{split}
\end{equation}
where the first term corresponds to $\delta^*_{ij}=0$ and the condition $\xij'\in \{0\}$, and the second term corresponds to $\delta^*_{ij}=\udelta$ and the condition $\xij'\in \{\lxij, \uxij\}$.
To minimize the second term, the original ReLU network only needs to be optimized towards the nearest feasible optimal solution.
It is easy to see from Proposition~\ref{prop:opt} that if $r(x,\delta_0^*,W,b)=0$, then $p^*_{\popt}=p^*_{\ropt}$, where $\ropt$ could be both Fast-Lin or CROWN.

Compared with $d(x,\delta_0^*,W,b)$, $ r(x,\delta_0^*, W, b)$ puts more constraints on the parameters $W,b$ , since it requires all unstable neurons of the ReLU network to match the optimal solutions of Fast-Lin, instead of only matching the objective values $p'_{\popt}$ and $p^*_{\ropt}$. 
In this way, it provides stronger guidance towards a network whose optimal solution for~\ref{eq:nonconvex} and Fast-Lin or CROWN agree.
However, again, this is not equivalent to trying to kill all unstable neurons, since Fast-Lin can be tight even when unstable neurons exist.

\subsection{Certified Robust Training in Practice}
In practice, for classification problems with more than two classes, we will compute the lower bound of the margins with respect to multiple classes. 
Denote $\mathbf{p}^*_{\ropt}$ and $\mathbf{p}^*_{\popt}$ as the concatenated vector of lower bounds of the relaxed problem and original problem for multiple classes, and $d_t,r_t$ as the regularizers for the margins with respect to class $t$. 

Together with the regularizers, we optimize the following objective
\begin{align}\label{eq:obj}
\underset{W,b}{\text{minimize }} L_{CE}(-\mathbf{p}^*_{\ropt}, y) + &\lambda \sum_{t}d_t(x,\delta_0^*, W,b) \nonumber \\
&\quad + \gamma \sum_{t}r_t(x,\delta_0^*, W,b),
\end{align}
where $L_{CE}(-\mathbf{p}^*_{\ropt}, y)$ is the cross entropy loss with label $y$, as adopted by many related works~\citep{wong2018scaling,gowal2018effectiveness}, and we have implicitly abbreviated the inner maximization problem w.r.t. $\{\delta_i\}_{i=0}^{L-1}$ into the optimal values $\mathbf{p}^*_{\ropt}$ and solution $\delta_0^*$. 
More details for computing the intermediate and output bounds can be found in Algorithm~\ref{alg:opt}, where we have used $\lVert \cdot \rVert_{1,row}$ to denote row-wise $\ell_1$ norm, and $(\cdot)_{:,j}$ for taking the $j$-th column.

One major challenge of the convex relaxation approach is the high memory consumption. 
To compute the bounds $\lxi,\uxi$, we need to pass an identity matrix with the same number of diagonal entries as the total dimensions of the input images, which can make the batch size thousands of times larger than usual.
To mitigate this, one can adopt the random projection from~\cite{wong2018scaling}, which projects identity matrices into lower dimensions as $\cW_{i:1}R$ to estimate the norm of $\cW_{i:1}$. 
Such projections add noise/variance to $\lxi, \uxi$, and the regularizers are affected as well. 

\section{Experiments}
\begin{table*}[!ht]
\centering
\resizebox{0.8\linewidth}{!}{%
\begin{tabular}{llll|cc|rrr}
\hline
\textbf{Dataset} & \textbf{Model}  & Base Method &~~~~$\epsilon$~~ & \textbf{$\lambda$} & \textbf{$\gamma$} & \textbf{Rob. Err} & \textbf{PGD Err} & \textbf{Std. Err}  \\ \hline
 MNIST           & 2x100, Exact    &  CP    &  0.1             &           0        &          0        &     14.85\%     &   \textbf{10.9\%}  & \textbf{3.65\%}         \\ 
  MNIST           & 2x100, Exact   &  CP  &   0.1             &           2e-3        &          1        &     \textbf{13.32\%}     &   \textbf{10.9\%}   & 4.73\%        \\ 
 \hline
 MNIST           & \texttt{Small}, Exact   &   CP &   0.1             &           0        &          0        &     4.47\%     &   2.4\%   & 1.19\%         \\ 
 MNIST           & \texttt{Small}, Exact   &   CP &  0.1             &         5e-3       &        5e-1       &  \textbf{3.65\%}    &  \textbf{2.2\%}  &  1.09\%    \\ 
 MNIST           &    -       &   Best of PV$^1$  & 0.1             &         -       &      -      &  4.44\%    &  2.87\%  &  1.20\%    \\
 MNIST           &    -       &   Best of RS$^2$  &     0.1             &         -       &        -       &  4.40\%    &  3.42\%  &  \textbf{1.05\%}    \\
 \hline
 
  MNIST           & \texttt{Small}         &    CP &  0.1             &           0        &          0        &   4.47\%        &    \textbf{3.3\%}  &    \textbf{1.19\%}               \\ 
 MNIST           & \texttt{Small}          &    CP &  0.1             &           0       &        5e-1       &  \textbf{4.32\%}    &  3.4\%  & 1.51\%    \\ 
 \hline
 MNIST           & \texttt{Large}          &    DAI$^3$           &     0.1             &           -       &        -       &     3.4\%      &   2.4\%   &  1.0\%     \\ \hline
  MNIST           & 2x100, Exact   &   CP  &  0.3             &           0        &          0        &     61.39\%     &   49.4\%   & 33.16\%         \\ 
  MNIST           & 2x100, Exact   &   CP  &  0.3             &           5e-3        &          5e-3        &     \textbf{56.05\%}     &   \textbf{44.3\%}   & \textbf{26.10\%}        \\ 
  \hline
 MNIST           & \texttt{Small}, Exact   &    CP  &    0.3             &           0        &          0        &   31.25\%      &     15.0\% &   7.88\%        \\ 
 MNIST           & \texttt{Small}, Exact   &     CP &0.3             &         5e-3       &        5e-1       &  \textbf{29.65\%}  &    \textbf{13.7\%}   &    \textbf{7.28\%}    \\ 
 \hline
 
  MNIST           & \texttt{Small}         &     CP  & 0.3             &           0        &          0        &   42.7\%         &  26.0\%      & 15.93\%               \\ 
 MNIST           & \texttt{Small}          &     CP  & 0.3             &         2e-3       &        2e-1       &  \textbf{41.36\%} &   \textbf{24.0\%}   & \textbf{14.29\%}    \\ \hline
   MNIST           & \texttt{XLarge}          &    IBP$^4$           &     0.3$^*$             &           -       &        -       &     8.05\%      &   6.12\%   &  \textbf{1.66\%}     \\ 
  MNIST           & \texttt{XLarge}          &    CROWN-IBP           &     0.3$^*$            &           -       &        -       &     7.01\%      &   5.88\%   &  1.88\%     \\ 
  MNIST           & \texttt{XLarge}          &    CROWN-IBP             &     0.3$^*$            &           0       &        5e-1       &    \textbf{6.64\%}      &   -    &  1.76\%     \\ \hline
 CIFAR10           & \texttt{Small}          &   CP  &  2/255          &           0        &          0        &      53.19\%      &    48.0\%   &    38.19\%    \\ 
 CIFAR10           & \texttt{Small}          &   CP  & 2/255          &         5e-3       &        5e-1       & \textbf{51.52\%}  &  \textbf{47.0\%}    & \textbf{37.30\%} \\ \hline    
   CIFAR10           & \texttt{Large}     &    DAI$^3$          &    2/255          &         -      &        -       & 61.4\%  & 55.6\% & 55.0\% \\ 
   CIFAR10           & \texttt{XLarge}     &   IBP$^4$          &    2/255$^*$          &         -      &        -       & 49.98\%  & 45.09\% & 29.84\% \\ 
   CIFAR10           & \texttt{XLarge}     &   CROWN-IBP          &    2/255$^*$          &         -      &        -       & 46.03\%  & 40.28\% & 28.48\% \\ 
  CIFAR10           & \texttt{Large}          &   CP  & 2/255          &           0        &          0        &      45.78\%      &  \textbf{38.5\%}  &   \textbf{29.42\%}    \\ 
 CIFAR10           & \texttt{Large}          &   CP  & 2/255          &         5e-3       &        5e-1       &     \textbf{45.19\%}   &  38.8\%  &  29.76\% \\ 
  \hline   
 
 CIFAR10           & \texttt{Small}          &   CP  & 8/255          &           0        &          0        &      75.45\%  &  68.3\%     &      62.79\%    \\ 
 CIFAR10           & \texttt{Small}          &   CP  & 8/255          &         1e-3       &        1e-1       & \textbf{74.70\%}   &  \textbf{67.9\%} & \textbf{62.50\%} \\ \hline         
  CIFAR10           & \texttt{Large}          &   CP  & 8/255          &           0        &          0        &      74.04\%  &   68.8\%      &     \textbf{59.73\%}    \\ 
 CIFAR10           & \texttt{Large}          &   CP  & 8/255          &         5e-3       &        5e-1       &   \textbf{73.74\%}  & \textbf{68.5\%}      & 59.82\%  \\ \hline   
 CIFAR10           & \texttt{XLarge}          &    IBP$^4$           &        8/255$^*$             &           -       &        -       &     67.96\%      &   \textbf{65.23\%}   &  \textbf{50.51\%}     \\ 
  CIFAR10           & \texttt{XLarge}          &    CROWN-IBP           &     8/255$^*$            &           -       &        -       &    66.94\%      &   65.42\%   &  54.02\%     \\ 
  CIFAR10           & \texttt{XLarge}          &    CROWN-IBP             &     8/255$^*$            &           0       &        5e-1       &    \textbf{66.64\%}      &   -    &  53.78\%     \\ \hline
\end{tabular}
}
\caption{ Results on MNIST, and CIFAR10 with small networks, large networks,
and different coefficients of $\regd$, $\regr$. All entries with positive $\lambda$ or $\gamma$ are using our regularizers. For all models not marked as ``Exact'', we have projected the input dimension of $\cW_{i:1}$ to 50, the same as \cite{wong2018scaling}. For $\epsilon$ values with $^*$, larger $\epsilon$ is used for training. $\epsilon=0.3,2/255,8/255$ correspond to using $\epsilon=0.4,2.2/255,8.8/255$ for training respectively. For the methods: $^1$:~\citep{dvijotham2018training}; $^2$:~\citep{xiao2018training}; $^3$:~\citep{mirman2018differentiable}; $^4$~\citep{gowal2018effectiveness}.   }
\label{table:results}
\end{table*}

We evaluate the proposed regularizer on two datasets (MNIST and CIFAR10) with two different $\epsilon$ each.
We consider only $\ell_\infty$ adversaries. 
Our implementation is based on the code released by~\cite{wong2018scaling} for Convex Outer Adversarial Polytope (CP), and~\cite{zhang2019towards} for CROWN-IBP, so when $\lambda=\gamma=0$, we obtain the same results as CP or CROWN-IBP.
We use up to 4 GTX 1080Ti or 2080Ti for all our experiments.

\paragraph{Architectures: } 
We experiment with a variety of different network structures, including a MLP (2x100) with two 100-neuron hidden layers as~\citep{salman2019convex}, two Conv Nets (\texttt{Small} and \texttt{Large}) that are the same as~\citep{wong2018scaling}, a family of 10 small conv nets and a family of 8 larger conv nets, all the same as~\citep{zhang2019towards}, 
and also the same 5-layer convolutional network (\texttt{XLarge}) as in the latest version of CROWN-IBP~\citep{zhang2019towards}. 

The \texttt{Small} convnet has two convolutional layers of 16, 32 output channels each and two FC layers with 100 hidden neurons. 
The \texttt{Large} convnet has four Conv layers with 32, 32, 64 and 64 output channels each, plus three FC layers of 512 neurons. 
The \texttt{XLarge} convnet has five conv lyaers with 64, 64, 128, 128, 128 output channels each, with two FC layers of 512 neurons.

\paragraph{Hyper-parameters:}
For experiments on CP, we use Adam~\citep{kingma2014adam} with a learning rate of $10^{-3}$ and no weight decay. 
Like \citep{wong2018scaling}, we train the models for 80 epochs, where in the first 20 epochs the learning rate is fixed but the $\epsilon$ increases from 0.01/0.001 to its maximum value for MNIST/CIFAR10, 
and in the following epochs, we reduce learning rate by half every 10 epochs. 
Unless labelled with ``Exact" in the model names of Table~\ref{table:results}, we use random projection as in~\cite{wong2018scaling} for CP experiments to reduce the memory consumption.
Due to the noisy estimation of the optimal solutions from these random projections, we also adopt a warm-up schedule for the regularizers in all CP experiments to prevent over-regularization, where $\lambda, \gamma$ increases form 0 to the preset values in the first 20 epochs. 

For CROWN-IBP, we use the updated expensive training schedule as~\citep{zhang2019towards}, which uses 200 epochs with batch size 256 for MNIST and 3200 epochs with batch size 1024 for CIFAR10. 
We also use the afore-mentioned warm up schedule for $\lambda, \gamma$.

\begin{table*}[h!]
\centering
\scalebox{0.7}{
 \begin{tabular}{c|c|c|c|c|c|c|c|c|c|c|c} 
\toprule
Method&error&model A&model B&model C&model D&model E&model F&model G&model H&model I&model J\\\midrule
\multirow{2}{*}{Copied}&std. (\%)&$5.97\pm.08$&$3.20\pm0$&$6.78\pm.1$&$3.70\pm.1$&$3.85\pm.2$&$3.10\pm.1$&$4.20\pm.3$&$2.85\pm.05$&$3.67\pm.08$&$2.35\pm.09$\\
&verified (\%)&$15.4\pm.08$&$10.6\pm.06$&$16.1\pm.3$&$11.3\pm.1$&$11.7\pm.2$&$9.96\pm.09$&$12.2\pm.6$&$9.90\pm.2$&$11.2\pm.09$&$9.21\pm.3$\\\hline
\multirow{2}{*}{Baseline}&std.  (\%)&$5.65\pm.04$&$3.23\pm.3$&$4.70\pm.4$&$2.94\pm.05$&$6.39\pm.3$&$2.89\pm.05$&$4.11\pm.3$&$2.55\pm.1$&$3.30\pm.4$&$2.56\pm.1$\\
&verified  (\%)&$14.70\pm.07$&$10.65\pm.3$&$13.78\pm1.$&$9.89\pm.3$&$15.26\pm.6$&$9.54\pm.1$&$12.06\pm.9$&$8.92\pm.2$&$10.81\pm.8  $&$9.67\pm.4$\\\hline
\multirow{2}{*}{With $r$}&std.  (\%)&$5.80\pm.04$&$3.16\pm.06$&$5.16\pm.3$&$3.06\pm.05$&$6.15\pm.2$&$2.95\pm.05$&$3.83\pm.2$&$2.57\pm.1$&$3.29\pm.04$&$2.73\pm.4$\\
&verified  (\%)&$14.54\pm.07$&$10.46\pm.08$&$13.37\pm.3$&$9.91\pm.3$&$14.43\pm.3$&$9.48\pm.2$&$11.01\pm.5$&$8.83\pm.3$&$10.18\pm.5$&$9.49\pm.4$\\
\bottomrule
 
 \end{tabular}
 }
 \caption{{Mean and standard deviation of the family of 10small models on MNIST with $\epsilon=0.3$. Here we use a cheaper training schedule with a total of 100 epochs, all in the same setting as the IBP baseline results of~\cite{zhang2019towards}. Baseline is CROWN-IBP with epoch=140 and lr\_decay\_step=20. Like in CROWN-IBP, we run each model 5 times to compute the mean and standard deviation. ``Copied" are results from~\citep{zhang2019towards}.}}
 \label{tab:stable}
\end{table*}

\subsection{Improving Convex Outer Adversarial Polytope} 
Table~\ref{table:results} shows comparisons with various approaches.
All of our baseline implementations of CP have already improved upon~\citep{wong2018scaling}. 
After adding the proposed regularizers, the certified robust accuracy is further improved upon our baseline in all cases. 
We also provide results against a 100-step PGD adversary for our CP models.
Since both PGD errors and standard errors are reduced in most cases, the regularizer should have improved not only the certified upper bound, but also improved the actual robust error. 

Despite the fact that we start from a stronger baseline, the relative improvement on 2x100 with our regularizer (10.3\%/8.7\%) are comparable to the improvements (5.9\%/10.0\%) under the same setting from~\citep{salman2019convex}, which  solves for the lower and upper bounds of all intermediate layers via the tightest layer-wise LP relaxation (Figure~\ref{fig:relaxation}).
This indicates that the improvement brought by using our regularizer during training and efficient verifiers (Fast-Lin in this case) for verification is comparable with using the expensive and unstable optimal layer-wise convex relaxation.

Our results with \texttt{Small} are better than the best results of~\citep{dvijotham2018training, xiao2018training} on MNIST with $\epsilon=0.1$, though not as good as the best of ~\citep{mirman2018differentiable}, which uses a larger model.
When applying the same model on CIFAR10, we achieve better robust error than~\citep{mirman2018differentiable}.

The relative improvements in certified robust error for $\epsilon=0.1$ and 0.3 are 18\%/3.4\% for the small exact model on MNIST, compared with 0.03\%/3.13\% for the random projection counterparts. 
This is mainly because in the exact models, we have better estimates of $\lxi,\uxi$. Still, these consistent improvements validate that our proposed regularizers improve the performance.

We also give ablation studies of the two regularizers in Appendix~\ref{sec:ablation} and Table~\ref{tab:ablation}.

\subsection{Improving CROWN-IBP}
In its first stage of training, CROWN-IBP~\cite{zhang2019towards} trains the network with CROWN~\cite{zhang2018efficient} to compute the bounds of the final outputs based on the interval bounds of intermediate activations, and in its second stage, CROWN-IBP uses only IBP.
We apply our regularizer to the first stage of CROWN-IBP, using interval bounds to over-approximate $\underline{x}_i$ and $\bar{x}_i$ required by our second regularizer on the optimal pre-activations, to obtain a better intialization for its second stage, and demonstrate improvements in Table~\ref{table:results}.

Methods based on interval bounds, including IBP, CROWN-IBP and DAI~\cite{mirman2018differentiable}, tend to behave not as good as CP when $\epsilon$ is small. 
Our regularizers are able to further improve CP on CIFAR10 ($\epsilon=2/255$), and demonstrate the best result among all approaches compared in this setting, as shown in Table~\ref{table:results}.
To our knowledge, these are the best results for CIFAR10 ($\epsilon=2/255$) reported on comparable sized models. 
By using our regularizers on CROWN-IBP to provide a better initialization for the later training stage of pure IBP, our method also achieves the best certified accuracy on MNIST ($\epsilon=0.3$) and CIFAR10 ( $\epsilon=8/255$).


To verify the significance of the regularizers, Table~\ref{tab:stable} shows the mean and variance of the results with the family smaller models on MNIST, demonstrating consistent improvements of our model, while Table~\ref{tab:mnist_large} (in the appendix) gives the best, median and worst case results with the large models on the MNIST dataset and compares with both IBP and CROWN-IBP.



\section{Conclusions}
We propose two regularizers based on the convex relaxation bounds for training robust neural networks that can be better verified by efficient verifiers including Fast-Lin and IBP for certifiable robustness. Extensive experiments validate that the regularizers improve robust accuracy over non-regularized baselines, and outperform state-of-the-art approaches. This work is a step towards closing the gap between certified and empirical robustness. Future directions include methods to improve computational efficiency for LP relaxations (and certified methods in general), and better ways to leverage random projections for acceleration. 

\bibliography{refs.bib}
\bibliographystyle{icml2020}

\clearpage
\appendix
\twocolumn[\icmltitle{Improving the Tightness of Convex Relaxation Bounds \\
for Training Certifiably Robust Classifiers \\ (Appendix)}]

\section{Proof of Proposition 1}
\begin{repproposition}\label{prop:opt}
Assume $\{z_i', x_i'\}_{i=1}^L$ is obtained by the ReLU network $h_L$ with input $z_1'$, and $\{\delta_{i}^*\}_{i=0}^{L-1}$ is the optimal solution of Fast-Lin or CROWN. If $z'_1=x+\delta_0^*$, and $x'_{ij}\in \mathcal{S}(\delta^*_{ij})$ for all $i=1,...,L-1, j\in \cI_i$, then $\{z_i', x_i'\}_{i=1}^{L}$ is an optimal solution of ~\ref{eq:nonconvex}, Fast-Lin and CROWN. 
Here 
\begin{equation*}
 \mathcal{S}(\delta^*_{ij})=\begin{cases}
 \{\lxij, \uxij\} & \text{if } \delta^*_{ij} = \udelta,\\
 \{0\} & \text{if } \delta^*_{ij} = 0.
 \end{cases}
\end{equation*}
\end{repproposition}

\begin{proof}
We only need to prove $\{z_i', x_i'\}_{i=1}^L$ is an optimal solution of both Fast-Lin and CROWN. 
After that, $\{z_i', x_i'\}_{i=1}^L$ is both a lower bound and feasible solution of~\ref{eq:nonconvex}, and therefore is the optimal solution of~\ref{eq:nonconvex}.

Here we define $x_i^*=W_i z^*_i + b_i$ for $i=1,...,L$, $z_{i+1}^*=D_i^{(L)} x_i^*+\delta_{i}^*$ for $i=1,...,L-1$, and $z_1^*=x+\delta_0^*$. 
By definition, $\{x_i^*, z_i^*\}_{i=1}^L$ is an optimal solution of Fast-Lin or CROWN.
Also, since $z^*_1=z'_1$, we have $x'_1=W_1 z^*_1+b_1=x^*_1$. 
Next, we will prove if the assumption holds, we will have $z'_2=z^*_2$ for both Fast-Lin and CROWN.

For $j\in \cI_1^+$, by definition of Fast-Lin and CROWN, $D^{(L)}_{1j}=1$ and $x^*_{1j}\ge \underline{x}_{1j} > 0$, so $x'_{1j}=x^*_{1j}\ge 0$, $z^*_{2,j}=D^{(L)}_{1j}x^*_{1j}=x'_{1j}=\max(x'_{1j},0)=z'_{2j}$.

For $j\in \cI_1^-$, again, by definition, $D^{(L)}_{1j}=0$, and $x^*_{1j}\le \bar{x}_{1j} < 0$, so $x'_{1j}=x^*_{1j}<0$, $z^*_{2j}=D^{(L)}_{1j}x^*_{1j}=0=\max(x'_{1j},0)=z'_{2j}$.

For $j\in \cI_1$:  
\begin{itemize}
    \item If $\delta^*_{1j}=0$ and $x'_{1j}=0$ as assumed in the conditions, since $z^*_1=z'_1$, we know
\begin{equation*}
    x^*_{1j}=W_{1j}z^*_1+b_{1j}=W_{1j}z'_{1} + b_{1j}=x'_{1j}=0,
\end{equation*}
where $W_{1j}$ is the $j$-th row of $W_1$.
No matter what value $D^{(L)}_{1j}$ is, $z^*_{2,j}=D^{(L)}_{1j}x^*_{1j}=0$, $z'_{2j}=\max(x'_{1j},0)=0$, the equality still holds.
    \item If $\delta^*_{1,j}=-\frac{\bar{x}_{1j}\underline{x}_{1j}}{\bar{x}_{1j}-\underline{x}_{1j}}$, for both Fast-Lin and CROWN, $z^*_{2j}=-\frac{\bar{x}_{1j}\underline{x}_{1j}}{\bar{x}_{1j}-\underline{x}_{1j}}(x^*_{1j}-\underline{x}_{1j})$. 
    Further, if $x'_{1j}\in \{\underline{x}_{1j}, \bar{x}_{1j}\}$ as assumed: if $x^*_{1j}=x'_{1j}=\underline{x}_{1j}$, then $x'_{1j}<0$, so $z'_{2j}=\max(x'_{1j},0)=0$, and $z^*_{2j}=\slope(\bar{x}_{1j}-\underline{x}_{1j})=0=z'_{2j}$; 
if $x^*_{1j}=x'_{1j}=\bar{x}_{1j}$, then $x'_{1j}>0$, $z'_{2j}=x'_{1j}$ and $z^*_{2j}=\slope(\bar{x}_{1j}-\underline{x}_{1j})=\uxtwoj=z'_{2j}$.
\end{itemize}

Now we have proved $z'_2=z^*_2$ for both Fast-Lin and CROWN if the assumption is satisfied. Starting from this layer, using the same argument as above, we can prove $z'_3=z^*_3$,...,$z'_{L-1}=z^*_{L-1}$ for both Fast-Lin and CROWN. 
As a result, $x'_L=x^*_L$ and $c_t^Tx'_L=c_t^Tx^*_L=p^*_{\ropt}$, where $\ropt$ can be both Fast-lin and CROWN.
Therefore, $\{z_i', x_i'\}_{i=1}^{L}$ is an optimal solution of $\ropt$.
\end{proof}


\section{Ablation Studies of the Two Regularizers}\label{sec:ablation}
\begin{table*}[h!]
\begin{center}
\scalebox{0.8}{
\begin{tabular}{r|r|r|r|r|r|r|r}
\hline
\multicolumn{1}{c|}{\textbf{$\lambda$}} & \multicolumn{1}{c|}{\textbf{$\gamma$}} & \textbf{Robust Err (\%)} & \textbf{Std. Err (\%)} & \multicolumn{1}{c|}{\textbf{$\lambda$}} & \multicolumn{1}{c|}{\textbf{$\gamma$}} & \multicolumn{1}{l|}{\textbf{Robust Err (\%)}} & \multicolumn{1}{l}{\textbf{Std. Err (\%)}} \\ \hline
1e-5                                     & 0                                      & 52.90                    & 37.61                      & 0                                       & 1e-3                                   & 52.56                                         & 37.45                                           \\ 
5e-5                                     & 0                                      & 53.09                    & 37.77                      & 0                                       & 5e-3                                   & 53.38                                         & 38.19                                           \\
1e-4                                     & 0                                      & 52.48                    & 37.37                      & 0                                       & 1e-2                                   & 52.60                                         & 37.73                                           \\ 
5e-4                                     & 0                                      & 52.85                    & 37.13                      & 0                                       & 2.5e-2                                 & 52.70                                         & 37.78                                           \\ 
1e-3                                     & 0                                      & 52.61                    & 37.96                      & 0                                       & 5e-2                                   & 53.13                                         & 38.36                                           \\ 
2.5e-3                                   & 0                                      & 53.10                    & 38.24                      & 0                                       & 1e-1                                   & 52.72                                         & 38.22                                           \\ 
5e-3                                     & 0                                      & 52.76                    & 38.15                      & 0                                       & 2.5e-1                                 & 52.90                                         & 38.04                                           \\ 
1e-2                                     & 0                                      & 53.14                    & 38.58                      & 0                                       & 5e-1                                   & 52.39                                         & 37.48                                           \\ 
5e-2                                     & 0                                      & 52.82                    & 39.89                      & 0                                       & 1                                      & \textbf{52.27}                                & \textbf{38.07}                                  \\ 
1e-1                                     & 0                                      & 53.94                    & 41.59                      & 0                                       & 2                                      & 53.10                                         & 38.64                                           \\ 
5e-1                                     & 0                                      & 56.39                    & 48.06                      & \multicolumn{1}{l|}{}                   & \multicolumn{1}{l|}{}                  & \multicolumn{1}{l|}{}                         & \multicolumn{1}{l}{}                           \\ 
1                                        & 0                                      & 59.23                    & 52.51                      & \multicolumn{1}{l|}{}                   & \multicolumn{1}{l|}{}                  & \multicolumn{1}{l|}{}                         & \multicolumn{1}{l}{}                           \\ \hline
\end{tabular}
}
\end{center}
\caption{Ablation results on CIFAR10 with the small model, where $\epsilon=2/255$.}\label{tab:ablation}
\end{table*}

In this section, we give the detailed results with either $\lambda$ or $\gamma$ set to 0, i.e., we use only one regularizer in each experiment, in order to compare the effectiveness of the two regularizers. 
All the results are with the small model on CIFAR10 with $\epsilon=2/255$.
The best results are achieved with $\regr$.
We reasoned in~\ref{sec:direct_indicator} that $\regd$ may not perform well when random projection is adopted. 
As shown in the supplementary, the best robust error achieved under the same setting when fixing $\gamma=0$ is higher than when fixing $\lambda=0$, which means $\regr$ is more resistant to the noise introduced by random projection. 
Still, random projections offer a huge efficiency boost when they are used. How to improve the bounds while maintaining efficiency is an important future work.

\section{Additional Results on MNIST}
See Table~\ref{tab:mnist_large}.
\begin{table*}[!t]
\centering
\resizebox{0.75\linewidth}{!}{%
\begin{tabular}{l|l|r|l|l|l|l|l|l|l}
\hline
\multirow{2}{*}{Dataset} & \multirow{2}{*}{$\epsilon(\ell_{\infty})$} & \multirow{2}{*}{Model Family}   & \multirow{2}{*}{Method} & \multicolumn{3}{l|}{Verified Test Error (\%)} & \multicolumn{3}{l}{Standard Test Error (\%)} \\
                         &                                            &                                 &                         & best         & median         & worst         & best         & median         & worst         \\ \hline\hline
\multirow{4}{*}{MNIST}   & \multirow{4}{*}{0.3}                       & \multicolumn{2}{l|}{\cite{gowal2018effectiveness}}               & 8.05         & -              & -             & 1.66         & -              & -             \\ \cline{3-10} 
                         &                                            & \multirow{3}{*}{8 large models} & CI Orig                 & 7.46         & 8.47           & 8.57          & 1.48         & 1.52           & 1.99          \\ \cline{4-10} 
                         &                                            &                                 & CI ReImp                & 7.99         & 8.38           & 8.97          & 1.40         & 1.69           & 2.19          \\ \cline{4-10} 
                         &                                            &                                 & CI Reg                  & 7.26         & 8.44           & 8.88          & 1.51         & 1.72           & 2.21          \\ \hline
\end{tabular}
}
\caption{\small Our results on the MNIST dataset, with CROWN-IBP. Here we use a cheaper training schedule with a total of 100 epochs, all in the same setting as the IBP baseline results of~\cite{zhang2019towards}. CI Orig are results copied from the paper, CI ReImp are results of our implementation of CROWN-IBP, and CI Reg is with regularizer $r$. }
\label{tab:mnist_large}
\end{table*}

\section{Solutions to the Relaxed Problems}\label{sec:cr_sols}
In this section, we give more details about the optimal solutions of Fast-Lin~\citep{weng2018towards} and CROWN~\citep{zhang2018efficient}, to make this paper self-contained.
Recall that for layer-wise convex relaxations, each neuron in the activation layer are independent. 
$\{\laij, \lbij, \uaij, \ubij\}$ are chosen to bound the activations assuming the lower bound $\underline{x}_{ij}$ and upper bound $\bar{x}_{ij}$ of the preactivation $x_{ij}$ is known. 
For $j\in \cI_i$, $\laij \xij + \lbij\le \max(0,\xij)\le \frac{\uxij}{\uxij - \lxij}(\xij-\lxij) \le \uaij\xij+\ubij$; for $j\in \cI_i^+$, $\laij=\uaij=\uaij=\ubij=1$; for $j\in \cI_i^-$, $\laij=\uaij=\lbij=\ubij=0$. 
\paragraph{Optimal Solutions of Fast-Lin}
In Fast-Lin~\citep{weng2018towards}, for $j\in \cI_i$, $\laij=\uaij=\frac{\uxij}{\uxij - \lxij}$, $\lbij=0$, $\ubij=-\frac{\uxij\lxij}{\uxij - \lxij}$, shown as the blue region in the right of Figure~\ref{fig:relaxation}.
To compute the lower and upper bound $\lxij$ and $\uxij$ for $x_{ij}$, we just need to replace the objective of Eq.~\ref{eq:convex_relaxation} with $c_{ij}^\top x_{i}$, where $c_{ij}$ is a one-hot vector with the same number of entries as $x_i$ and the $j$-th entry being 1 for $\lxij$ and -1 for $\uxij$ (an extra negation is applied to the minimum to get $\uxij$).

Such constraints allow each intermediate ReLU activation to reach their upper or lower bounds independently.
As a result, each intermediate {\em unstable neuron} can be seen as an adversary adding a perturbation $\delta_{ij}$ in the range $[0, -\frac{\uxij\lxij}{\uxij - \lxij}]$ to a linear transform, represented as $\zij=\frac{\uxij}{\uxij - \lxij}\xij + \delta_{ij}$. 
Such a point of view gives rise to a more interpretable explanation for Fast-Lin. 
If we construct a network from the relaxed constraints, then the problem becomes how to choose the perturbations for both the input and intermediate unstable neurons to minimize $c_t^\top x_{L}$ of a multi-layer linear network. 
Such a linear network under the perturbations is defined as
\begin{equation}\label{eq:lin_net}
    z_{i+1}=D_{i}x_i + \delta_i, x_i = W_i z_i+b_i,~\text{for } i=1,...,L, z_1=x+\delta_0,
\end{equation}
where $D_{i}$ is a diagonal matrix and $\delta_i$ is a vector. 
The input perturbation satisfies $\lVert \delta_0 \rVert_p \le \epsilon$.
The $j$-th diagonal entry $D_{ij}$ and the $j$-th entry $\delta_{ij}$ for $i>0$ is defined as 
\begin{align}\label{eq:Danddelta}
    D_{ij}&=\begin{cases}
     0, & \text{if}\ j\in \cI_i^- \\
     1, & \text{if}\ j\in \cI_i^+, \\
     \frac{\uxij}{\uxij - \lxij}, & \text{if} j \in \cI_i
    \end{cases}\nonumber\\
    \delta_{ij}&\in \cS_{\delta_{ij}} = 
    \begin{cases}
     \left[0, -\frac{\uxij\lxij}{\uxij - \lxij} \right], & \text{if}\ j\in \cI_i \\
      \{0\}, & \text{otherwise}.
    \end{cases}
\end{align}
With such an observation, we can further unfold the objective in Eq.~\ref{eq:convex_relaxation}
\ref{eq:lin_net} 
into a more interpretable form as
\begin{align}\label{eq:newop}
    c_t^\top x_L &=c_t^\top  f_L(D_{L-1}f_{L-1}(\cdots D_1f_1(x) )) \nonumber\\
    &\quad + c_t^\top \sum_{i=0}^{L-1} W_L\prod_{k=i+1}^{L-1} D_kW_k \delta_i, 
\end{align}
where the first term of RHS is a forward pass of the clean image $x$ through a linear network interleaving between a linear layer $x=W_iz+b_i$ and a scaling layer $z=D_i x$, and the second term is the sum of the $i$-th perturbation passing through all the weight matrices $W_i$ of the linear operation layers and scaling layers $D_i$ after it. 


Therefore, under such a relaxation, only the second term is affected by the variables $\{\delta_i\}_{i=0}^{L-1}$ for optimizing Eq.~\ref{eq:convex_relaxation}. 
Denote the linear network up to the $i$-th layer as $g_i(x)$, and $\mathcal{W}_{i:i'}=W_i\prod_{k=i'}^{i-1}D_kW_k$. 
We can transform Eq.~\ref{eq:convex_relaxation} with $\laij=\uaij=\frac{\uxij}{\uxij - \lxij}$, $\lbij=0$, $\ubij=-\frac{\uxij\lxij}{\uxij - \lxij}$ into the following constrained optimization problem
\begin{align}\label{eq:parallel_relax}
    \underset{\delta_0,...,\delta_{L-1}}{\text{minimize }} &c_t^\top g_L(x) + \sum_{i=0}^{L-1}c_t^\top \cW_{L:i+1}\delta_i,\nonumber\\
    \text{subject to}~&\lVert \delta_0 \rVert_p \le \epsilon, \delta_{ij} \in \cS_{\delta_{ij}}~\text{for }i=1,...,L-1.
\end{align}
Notice $c_t^\top \cW_{L:i+1}$ is just a row vector. 
For $i>0, j\in \cI_i$, to minimize $c_t^\top \cW_{L:i+1}\delta_{ij}$, we let $\delta_{ij}$ be its minimum value $0$ if $(c_t^\top\cW_{L:i+1})_j\ge 0$, or its maximum value $\delta_{ij}= -\frac{\uxij\lxij}{\uxij - \lxij}$ if $(c_t^\top\cW_{L:i+1})_j < 0$, where $(c_t^\top\cW_{L:i+1})_j$ is the $j$-th entry of $c_t^\top\cW_{L:i+1}$. 
For $\delta_0$, when the infinity norm is used, we set $\delta_{ij}=-\epsilon$ if $(c_t^\top\cW_{L:i+1})_j\ge 0$, and otherwise $\delta_{ij}=\epsilon$. 
For other norms, it is also easy to see that
\begin{align}\label{eq:dual_norm}
    \min_{\lVert \delta_0 \rVert_p\le \epsilon} c_t^\top\mathcal{W}_{L:1}\delta_0&=-\epsilon\max_{\lVert \delta_0 \rVert_p\le 1}-c_t^\top\mathcal{W}_{L:1}\delta_0\\
    &=-\epsilon \lVert c_t^\top\mathcal{W}_{L:1} \rVert_*,
\end{align}
where $\lVert\cdot \rVert_*$ is the dual norm of the $p$ norm.
In this way, the optimal value 
$p_\ropt^*$
$p_{\mathcal{C}}$ 
of the relaxed problem (Eq.~\ref{eq:parallel_relax}) can be found efficiently without any gradient step.
The optimal value can be achieved by just treating the input perturbations and intermediate relaxed ReLU activations as adversaries against a linear network after them.
The resulting expression for the lower-bound is
\begin{align}\label{eq:pc}
    p_\popt^* \ge p_\ropt^* &= c_t^\top  g_L(x) - \epsilon \lVert c_t^\top  \cW_{L:1} \rVert_* \nonumber\\
    &\quad- \sum_{i=1}^{L-1}\sum_{j\in I_i}\frac{\uxij\lxij}{\uxij - \lxij}\min((c_t^\top \cW_{L:i+1})_j, 0).
\end{align}

Though starting from different points of view, it can be easily proved that the objective derived from a dual view in~\citep{wong2018scaling} is the same as Fast-Lin. 

\paragraph{Optimal Solution of CROWN}
The only difference between CROWN and Fast-Lin is in the choice of $\underline{a}_{ij}$ for $j\in \mathcal{I}_i$.
For ReLU activations, CROWN chooses $\underline{a}_{ij}=1$ if $\bar{x}_{ij}\ge -\underline{x}_{ij}$, or $\underline{a}_{ij}=0$ otherwise. 
This makes the relaxation tighter than Fast-Lin, but also introduces extra complexity due to the varying $D_i$. 
In Fast-Lin, $D_i$ is a constant once the upper and lower bounds $\bar{x}_i$ and $\underline{x}_i$ are given. 
For CROWN, since $0<\bar{a}_{ij}<1$, $\bar{a}_{ij}\neq \underline{a}_{ij}$, $D_i$ now changes with the optimality condition of $\delta_i$, which depends on the layer $l$ and the index $k$ of the neuron/logit of interest. 
Specifically, for $\ell_\infty$ adversaries, the optimality condition of $\delta_i$ is determined by $c_{lk}^\top \mathcal{W}_{l:i}$, so now we have to apply extra index to the slope as $D_i^{(l,k)}$, as well as the equivalent linear operator as $\cW_{l:1}^{(l,k)}$. 
As a result, the optimal solution is now
\begin{align}
    c_{lk}^\top  g_l(x) &- \epsilon \lVert c_{lk}^\top  \cW_{l:1}^{(l,k)} \rVert_* \nonumber\\
    &\quad- \sum_{i=1}^{L-1}\sum_{j\in I_i}\frac{\uxij\lxij}{\uxij - \lxij}\min((c_{lk}^\top \cW^{(l,k)}_{l:i+1})_j, 0).
\end{align}
This drastically increase the number of computations, especially when computing the intermediate bounds $\underline{x}_i$ and $\bar{x}_i$, where we can no longer just compute a single $\cW_{l:1}$ to get the bound, but have to compute number-of-neuron copies of it for the different values of $D_i^{(l,k)}$ in the intermediate layers.

\paragraph{Practical Implementations of the Bounds}
In practice, the final output bound (also the intermediate bounds) is computed in a backward pass, since we need to determine the value $(c_{lk}^\top \cW^{(l,k)}_{l:i+1})_j$ to choose the optimal $\delta_{ij}^*$, which is the multiplication of all linear operators after layer $i$. 
Computing $c_{lk}^\top \cW^{(l,k)}_{l:i+1}$ in a backward pass avoids repeated computation.
It proceeds as
\begin{equation}
    c_{lk}^\top \cW^{(l,k)}_{l:i} = c_{lk}^\top \cW^{(l,k)}_{l:i+1} D_i^{(l,k)} W_i. 
\end{equation}

\section{A Toy Example for Tight Relaxation}\label{sec:toy}
\begin{figure}[!htbp]
\centering
  \includegraphics[width=.6\linewidth]{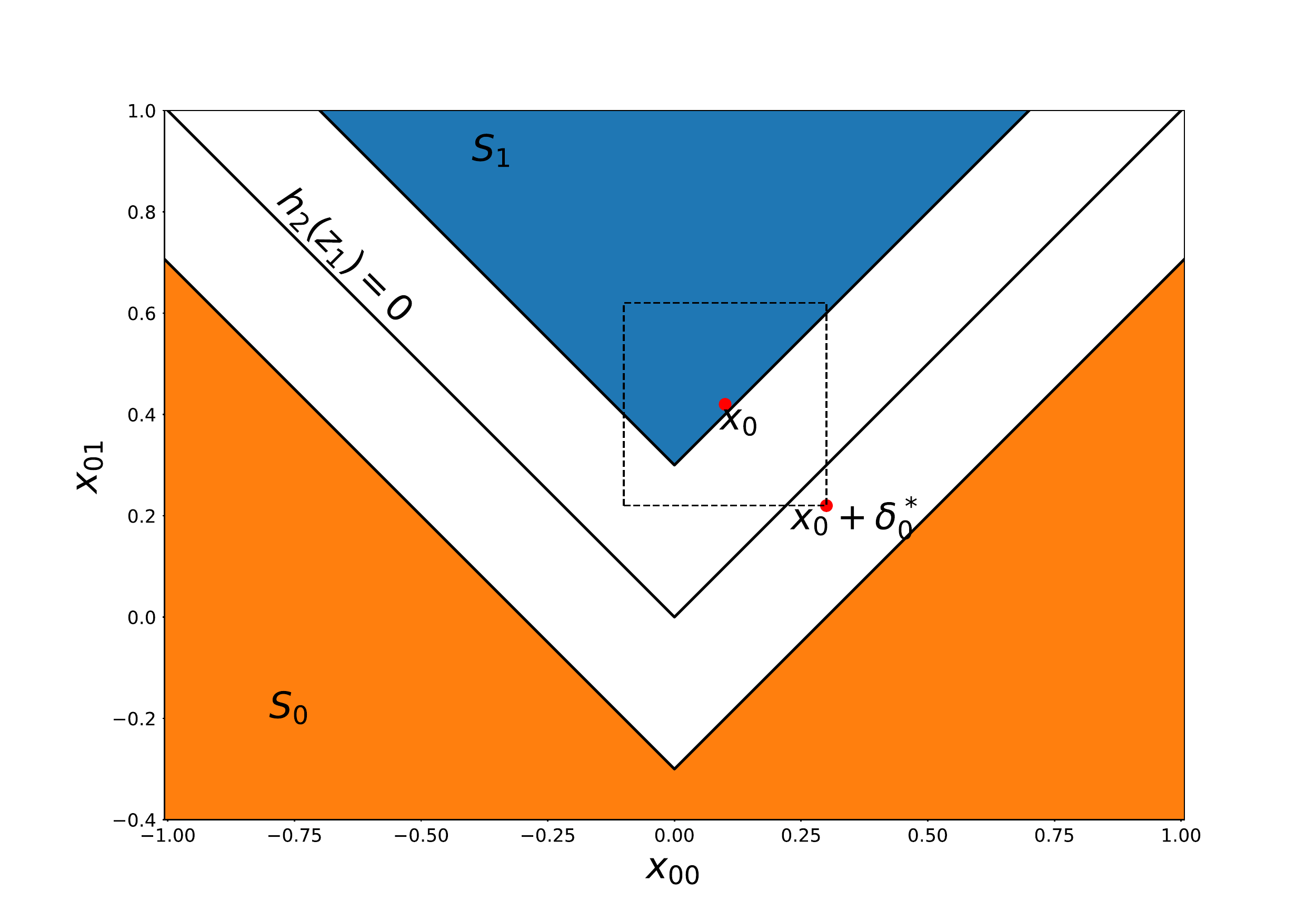}
  \caption{Illustration of the data distribution and the decision boundary of the network. In this case, $b=0.3, \epsilon=0.2, x_0=[0.1, 0.42]^T$.
  }
  \label{fig:toy} 
\end{figure}

We give an illustrative example where the optimal solution to the relaxed problem is a feasible solution to the original non-convex problem for certain input samples even when unstable neurons exist.
It is a binary classification problem for samples $x_0=[x_{01},x_{02}]^T\in \real^2$. 
We assume $x_0$ is uniformly distributed in $S_0\cup S_1$, where $S_0=\{x_0| x_{02}\le |x_{01}|-b, \lVert x_0 \rVert_{\infty}\le 1\}$, $S_1=\{x_0| x_{02}\ge |x_{01}|+b,  \lVert x_0 \rVert_{\infty}\le 1\}$, and $0<b<1$.
The ground-truth label for $x_0\in S_0$ and $x_0\in S_1$ are 0, 1 respectively.
The maximal-margin classifier for such data distribution is $\mathbbm{1}_{\{z_{12}\ge |z_{12}|\}}$, where $z_1$ is the input to the classifier. 
The data distribution and the associated maximal-margin classifier is shown in Figure~\ref{fig:toy}.

This maximal-margin classifier can be represented by a ReLU network with single hidden layer as $\mathbbm{1}_{\{h_2(z_1)\ge 0\}}$, where $h_2(z_1)=W_2\sigma(W_1z_1)$, and
\begin{equation}
    W_1=\begin{bmatrix}
    1  & 0 \\
    -1 & 0 \\
    0  & 1 \\
    0  & -1
    \end{bmatrix},
    W_2=\begin{bmatrix}
    -1 & -1 & 1 & -1
    \end{bmatrix}.
\end{equation}

\begin{claim}[Convex relaxation can be tight when unstable neurons exist]
The solution to the relaxed problem~\ref{eq:parallel_relax} is feasible for the original non-convex problem~\ref{eq:nonconvex} of the aforementioned ReLU network $h_2(x_0+\delta_0)=W_2\sigma(W_1(x_0+\delta_0))$ for any $x_0\in S_1$ under any perturbation $\delta_0\in \{ \delta|\lVert \delta \rVert_{\infty} \le \epsilon, 0<\epsilon<b \}$.
\end{claim}

\begin{proof}
Both the network and $S_1$ are symmetric in $x_{01}$, therefore it is sufficient to prove the result for $x_{01}\ge 0$.

\textbf{(1)} If $x_{01}\ge \epsilon$, since $x_{01}\in S_1$, $x_{02}\ge b+\epsilon$. 
The 4 neurons in $x_1=W_1 (x_0+\delta_0)$ are either non-negative or non-positive for any $\delta_0\in \{ \delta|\lVert \delta \rVert_{\infty} \le \epsilon, 0<\epsilon<b \}$ and the convex relaxation is tight.

\textbf{(2)} If $x_{01} < \epsilon$, for any input sample $x_0\in\{x|x=[a,c]^T, 0<a<b < b+a\le c\}\subseteq S_1$, the optimal perturbation $\delta_{\mathcal{O}}^* \in \{\delta|\lVert \delta \rVert_\infty \le \epsilon, 0 < a < \epsilon < b\}$ can be inferred from Figure~\ref{fig:toy} as $[\epsilon, -\epsilon]^T$.
The corresponding ReLU activations and the optimal solution are 
\begin{equation}\label{eq:z2_orig_toy}
z'_2=[a+\epsilon, 0, c-\epsilon, 0]^T, x'_2=c-a-2\epsilon.
\end{equation}

Meanwhile, for the relaxed problem, the lower and upper bounds of the hidden neurons $x_1=W_1z_1$ are 
\begin{align*}
    \underline{x}_1=[a-\epsilon, -a-\epsilon, c-\epsilon, -c-\epsilon]^T,\\
    \bar{x}_1=[a+\epsilon, -a+\epsilon, c+\epsilon, -c+\epsilon]^T.    
\end{align*}

Therefore, the first 2 hidden neurons are unstable neurons, and the convex relaxation we are using will relax the ReLU operation $z_2=\sigma(x_1)$ into $z_2=D_1 x_1+\delta_1$, where $D_1$ is a diagonal matrix, and $\delta_1$ are slack variables bounded by $0\le \delta_1 \le \bar{\delta}_1$. 
The diagonal entries of $D_1$ and the upper bounds $\bar{\delta}_1$ are defined by Eq.~\ref{eq:Danddelta} as
\begin{align*}
    \text{diag}(D_1) &= \left[\frac{a+\epsilon}{2\epsilon}, \frac{-a+\epsilon}{2\epsilon}, 1, 0\right], \\
    \bar{\delta}_1 &= \left[\frac{(-a+\epsilon)(a+\epsilon)}{2\epsilon}, \frac{(-a+\epsilon)(a+\epsilon)}{2\epsilon}, 0 , 0  \right]^T,
\end{align*}
i.e., $\delta_{13}$ and $\delta_{14}$ are always 0. 
The relaxed linear network, as defined by the constraints in Eq.~\ref{eq:convex_relaxation} with our specific relaxation, is now determined as $x_2=h_2(z_1)=W_2(D_1 W_1 (x+\delta_0)+\delta_1)$.
It can be written into the same form as Eq.~\ref{eq:newop} as 
\begin{align}\label{eq:app_obj}
    c_t x_2&=c_t h_2(z_1)=c_t W_2D_1 W_1 (x+\delta_0) + c_tW_2\delta_1\\
    &=c_t \mathcal{W}_{2:1} x + c_t \mathcal{W}_{2:1}\delta_0 + c_t W_2\delta_1,
\end{align}
where 
\begin{equation*}
    c_t=1, \mathcal{W}_{2:1}=[-\frac{a}{\epsilon},1].
\end{equation*} 
Therefore, to minimize the term $c_t \mathcal{W}_{2:1}\delta_0$, we should choose $\delta^*_{0}=[\epsilon, -\epsilon]^T$, which is equal to $\delta^*_{\mathcal{O}}$. 
To minimize $c_t W_2\delta_1$, we should let $\delta_1^*=\left[\frac{(-a+\epsilon)(a+\epsilon)}{2\epsilon}, \frac{(-a+\epsilon)(a+\epsilon)}{2\epsilon}, 0, 0\right]$, 
which gives rise to the optimal solution of the relaxed problem as
\begin{equation*}
    z_2^* = D_1W_1(x+\delta_0^*)+\delta_1^*=[a+\epsilon, 0, c-\epsilon, 0]^T, x_2^* = c-a-2\epsilon,
\end{equation*}
the same as the optimal solution of the original non-convex problem given in Eq.~\ref{eq:z2_orig_toy}.
This shows both of the regularizers, in this case instantiated as 
\begin{align*}
d(x,\delta_0^*, W_1,W_2)&=c_t(x_2'-x_2^*), r(x,\delta_0^*, W_1,W_2)\\
&= \frac{1}{2}\left(|x_{11}'-x_{11}^*|+|x_{12}'-x_{12}^*|\right), 
\end{align*}
are able to reach 0 for certain networks and samples when non-stable neurons exist.
\end{proof}

It might seem that adding $d(x,\delta_0^*, W,b)=p'_{\popt}(x,\delta_0^*)-p_{\ropt}^*$ as a regularizer into the loss function will undesirably minimize the margin $p'_{\popt}(x,\delta_0^*)$ for the ReLU network.
Theoretically, however, it is not the case, since $\min p'_{\popt}-p_{\ropt}^*$ is a different optimization problem from neither $\min p'_{\popt}$ nor $\max p_{\ropt}^*$.
In fact, the non-negative $d(x,\delta_0^*, W,b)$ could be minimized to 0 with both $p'_{\popt}$ and $p_{\ropt}^*$ taking large values.
In the illustrative example, it is easy to see that for any $x_0\in \{x|x=[a,c]^T, 0<a<b < b+a\le c\}$, $d(x,\delta_0^*, W,b)=c_t(x_2'-x_2^*)=0$, but $p'_{\popt}=p_{\ropt}^*=c-a-2\epsilon>0$ when $\epsilon < \frac{b}{2}$.

Moreover, since we are maximizing $p_{\ropt}^*$ via the robust cross entropy loss\footnote{The cross entropy loss on top of the lower bounds of margins to all non-ground-truth classes, see Eq.~\ref{eq:app_obj}.} while minimizing the non-negative difference $p'_{\popt}-p_{\ropt}^*$, the overall objective tends to converge to a state where both $p'_{\popt}$ and $p_{\ropt}^*$ are large. 

\section{Difficulties in Adapting IBP for $\ell_2$ Adversary}\label{sec:ibp_l2}
The Inverval Bound Propagation (IBP) method discussed here is defined in the same way as~\citep{gowal2018effectiveness}, where the bound of the margins are computed layer-wise from the input layer to the final layer, and the bound of each neuron is considered independently for both bounding that neuron and using its inverval to bound other neurons. 

It is natural to apply IBP against $\ell_\infty$ adversaries, since each neurons are allowed to change independently in its interval, which is similar to the $\ell_\infty$ ball.
One way to generalize IBP to other $\ell_p$ norms is to modify the bound propagation in the first layer, such that any of its output neuron ($i$) is bounded by an interval centered at $x_{1i}=W_{1,i}x+b_{1,i}$ with a radius of $\epsilon_p \lVert W_{1,i} \rVert_{p^*}$, where $x_{1i}$ the clean image $x$'s response, and $W_{1,i}$ is the first layer's linear transform corresponding to the neuron, and $\lVert \cdot \rVert_{p^*}$ is the dual norm of $\lVert \cdot \rVert_{p}$, with $\frac{1}{p}+\frac{1}{p^*}=1$.
We refer to this approach IBP($\ell_p$, $\epsilon_p$).
Here by the example of $\ell_2$ norm, we show such an adaptation may not be able to obtain a robust \textit{convolutional} neural network compared with established results, such as reaching 61\% certified accuracy on CIFAR10 with $\epsilon_2=0.25$~\citep{cohen2019certified}.

Specifically, for adversaries within the $\ell_2$-ball $\mathcal{B}_{2,\epsilon_2}(x)$, IBP($\ell_2$, $\epsilon_2$) computes the upper and lower bounds as
\begin{align}\label{eq:ibp_l2}
    \bar{x}_{1i}^2 = W_{1,i} x + \epsilon_2 \lVert W_{1,i}\rVert_2 + b_{1,i},\nonumber\\
    \underline{x}_{1i}^2 = W_{1,i} x - \epsilon_2 \lVert W_{1,i}\rVert_2 + b_{1,i}.
\end{align}
By comparison, for some adversary within the $\ell_\infty$-ball $\mathcal{B}_{\infty,\epsilon_\infty}(x)$, IBP($\ell_\infty$, $\epsilon_\infty$) computes the upper and lower bounds as 
\begin{align}\label{eq:ibp_linf}
    \bar{x}_{1i}^{\infty} = W_{1,i} x + \epsilon_\infty \lVert W_{1,i}\rVert_1 + b_{1,i},\nonumber\\
    \underline{x}_{1i}^{\infty} = W_{1,i} x - \epsilon_\infty \lVert W_{1,i}\rVert_1 + b_{1,i}.
\end{align}
Since the two approaches are identical in the following layers, to analyze the best-case results of IBP($\ell_2$, $\epsilon_2$) based on established results of IBP($\ell_\infty$, $\epsilon_\infty$), it suffices to compare the results of IBP($\ell_\infty$, $\epsilon_\infty$) with $\epsilon_\infty$ set to some value such that the range $\bar{x}_{1}^{\infty}-\underline{x}_{1}^{\infty}$ of Eq.~\ref{eq:ibp_linf} is majorized by the range $ \bar{x}_{1}^2  - \underline{x}_{1}^2$ of Eq.~\ref{eq:ibp_l2}. 
In this way, we are assuming a weaker adversary for IBP($\ell_\infty$, $\epsilon_\infty$) than the original IBP($\ell_2$, $\epsilon_2$), so its certified accuracy is an upper bound of IBP($\ell_2$, $\epsilon_2$).
Therefore, it suffices to let
\begin{equation}\label{eq:range_cond}
    \epsilon_\infty \lVert W_{1,i}\rVert_1 \le \epsilon_2 \lVert W_{1,i}\rVert_2,~ \forall i=1,2,...,n_1.
\end{equation}
For any $W_{1,i}\in \real^d$, we have $\lVert W_{1,i} \rVert_1 \le \sqrt{d} \lVert  W_{1,i} \rVert_2 $. 
To make Eq.~\ref{eq:range_cond} hold for any $W_{1,i}\in \real^d$, we can set
\begin{equation*}
    \epsilon_\infty = \frac{1}{\sqrt{d}}\epsilon_2.
\end{equation*}
In general, $d$ is equal to the data dimension, such as 3072 for the CIFAR10 dataset. 
However, for convolutional neural networks, the first layer is usually convolutional layers and $W_{1,i}$ is a 3072-dimensional sparse vector with at most $k\times k\times 3$ non-zero entries at fixed positions for convolution kernels with size $k$ and input images with 3 channels.
In~\citep{zhang2019towards, gowal2018effectiveness, wong2018scaling}, $k=3$ for their major results.
In this case, 
\begin{equation}
\epsilon_\infty \ge \frac{1}{3\sqrt{3}}\epsilon_2. 
\end{equation}

Under such assumptions, for $\epsilon_2=0.25$, the certified accuracy of IBP($\ell_2$, 0.25) on CIFAR10 should be upper bounded by IBP($\ell_\infty$, 0.04811), unless changing the first layer bounds into $\ell_\infty$ norm based bounds significantly harms the performance.\footnote{Which is unlikely, since $\epsilon_\infty$ is now a weaker adversary than $\epsilon_2$ and the difference in gradient expression only appears in the first layer.} 
The best available results of certified accuracies are 33.06\% for IBP($\ell_\infty$, 0.03137) and 23.20\% for IBP($\ell_\infty$, 0.06275)~\citep{zhang2019towards}. 
Comparing with the established results from~\citep{cohen2019certified} (61\%), we can conclude the certified accuracy of IBP($\ell_2$, 0.25) is at least 27.93\% to 37.80\% lower than the best available results, since we are assuming a weaker adversary. 

IBP($\ell_2$, $\epsilon_2$) is also not as good as the results with convex relaxation from~\citep{wong2018scaling}, where the best single-model (with projection as approximation) certified accuracy with $\epsilon_2=36/255$ is 51.09\%.
For IBP, this adversary is no weaker than $\epsilon_\infty=6.9282/255$. 
The best available results for IBP($\ell_\infty$, 2/255) and IBP($\ell_\infty$, 8/255) are 50.02\%~\citep{gowal2018effectiveness} and 33.06\%~\citep{zhang2019towards} respectively, which indicates the certified accuracy of IBP($\ell_2$, 36/255) is at least 1.07\% to 18.03\% worse (much loser to 18.03\%) than the approximated version of convex relaxation under the same $\ell_2$ adversary.

\end{document}